\documentclass[conference]{IEEEtran}
\usepackage{graphicx,times,psfig,amsmath,amscd,amssymb,mathrsfs} % Add all your packages here

\newtheorem{theorem}{Theorem}%[section]
 
\newtheorem{lemma}[theorem]{Lemma}

\newtheorem{example}[theorem]{Example}
\newtheorem{remark}[theorem]{Remark}

\newcommand{\e}{{\epsilon}}

\newcommand{\VC}{{\mbox{VC}}}

\newcommand{\abs}[1]{\lvert#1\rvert}

\def\N{{\mathbb N}}

\def\R{{\mathbb R}}

\newcommand{\neswarrow}{\mathrel{\text{$\nearrow$\llap{$\swarrow$}}}}

\IEEEoverridecommandlockouts	% to create the author's affliation portion
\begin{document}

% paper title: Must keep \ \\ \LARGE\bf in it to leave enough margin.
\title{\ \\ \LARGE\bf 
PAC learnability versus VC dimension: a footnote to a basic result of statistical learning
\thanks{Vladimir Pestov is with Departamento de Matem\'atica,
Universidade Federal de Santa Catarina,
Campus Universit\'ario Trindade,
CEP 88.040-900 Florian\'opolis-SC, Brasil (CNPq Visiting Researcher) and
the Department of Mathematics and Statistics, University of Ottawa, 585 King Edward Avenue, Ottawa, Ontario, K1N 6N5 Canada (permenent address, phone: 613-562-5800 ext. 3523, fax: 613-562-5776, email: vpest283@uottawa.ca).}}

\author{Vladimir Pestov}
\maketitle

\begin{abstract} 
A fundamental result of statistical learnig theory states that a concept class is PAC learnable if and only if it is a uniform Glivenko--Cantelli class if and only if the VC dimension of the class is finite. However, the theorem is only valid under special assumptions of measurability of the class, in which case the PAC learnability even becomes consistent. Otherwise, there is a classical example, constructed under the Continuum Hypothesis by Dudley and Durst and further adapted by Blumer, Ehrenfeucht, Haussler, and Warmuth, of a concept class of VC dimension one which is neither uniform Glivenko--Cantelli nor consistently PAC learnable. We show that, rather surprisingly, under an additional set-theoretic hypothesis which is much milder than the Continuum Hypothesis (Martin's Axiom), PAC learnability is equivalent to finite VC dimension for every concept class.
\end{abstract}

\section{Introduction}
The following is a fundamental result of statistical learning theory.

\begin{theorem}
\label{th:fundamental}
For a concept class $\mathscr C$ the following three conditions are equivalent: 
\begin{enumerate}
\item
$\mathscr C$ is distribution-free PAC learnable, 
\item $\mathscr C$ is a uniform Glivenko--Cantelli class, and 
\item the Vapnik--Chervonenkis dimension of $\mathscr C$ is finite.
\end{enumerate}
\end{theorem}
\smallskip

It is in this form that the theorem is usually stated in textbooks on the subject, see \cite{VC1971,BEHW}. The condition 1) means the existence of a learning rule for $\mathscr C$ which is probably approximately correct.

However, strictly speaking, the result is only true under a suitable measurability assumption on the concept class $\mathscr C$. One such assumption is that of $\mathscr C$ being {\em image admissible Souslin}: the class $\mathscr C$ can be parametrized with elements of the unit interval so that pairs $(x,t)$, $x\in C_t$, $t\in [0,1]$ form an analytic subset of $\Omega\times [0,1]$ \cite{dudley}. Another measurability assumption, more difficult to state, is that of a {\em well-behaved class} $\mathscr C$ \cite{BEHW}. Under either of those conditions, the statement (1) in Theorem \ref{th:fundamental} can be replaced with

\smallskip
1$^\prime$) $\mathscr C$ is distribution-free consistently PAC learnable,
\smallskip

\noindent
meaning that every consistent learning rule $\mathcal L$ for $\mathscr C$ is distribution-free probably approximately correct. In the proof, a measurability hypothesis on $\mathscr C$ has to be invoked twice, in order to deduce implications (3) $\Rightarrow$ (2) and (1$^\prime$ $\Rightarrow$ (1).

In particular, Theorem \ref{th:fundamental} holds for every countable class $\mathscr C$ or, more generally, for every {\em universally separable} class \cite{pollard}. It is arguable that every concept class emerging in either theory or applications of statistical learning will be measurable in a sufficiently strong sense. For this reason, a measurability condition on $\mathscr C$ is typically not even mentioned.

The fact remains that Theorem \ref{th:fundamental} cannot be derived in full generality. An example of a concept class $\mathscr C$ of Vapnik--Chervonenkis (VC) dimension one which is not uniform Glivenko--Cantelli was constructed by Durst and Dudley \cite{DD}, and a further modification of this example, also of VC dimension one, fails consistent PAC learnability \cite{BEHW}. 

This example has been constructed under Continuum Hypothesis (CH), which is arguably not a natural assumption in a probabilistic context \cite{freiling}. However, the example remains valid under much more relaxed and natural set-theoretic hypothesis: Martin's Axiom (MA). It is one of the most often used and best studied additional set-theoretic assumptions beyond the standard Zermelo-Frenkel set theory with the Axiom of Choice (ZFC). In particular, Martin's Axiom follows from the Continuum Hypothesis (CH), but it is also compatible with the negation of CH, and in fact it is namely the combination MA+$\neg$CH that is really interesting  \cite{fremlin,jech,kunen}.

In this note we make the following, somewhat astonishing, observation: under the same assumption (Martin's Axiom), the conditions (1) and (3) in Theorem \ref{th:fundamental} are equivalent. Here is our main result.

\begin{theorem}
\label{th:main}
Assume the validity of Martin's Axiom (MA). Then the following are equivalent for every concept class $\mathscr C$ consisting of universally measurable subsets of a Borel domain $\Omega$:
\begin{enumerate}
\item
$\mathscr C$ is distribution-free PAC learnable, and 
\item the Vapnik--Chervonenkis dimension of $\mathscr C$ is finite.
\end{enumerate}
\end{theorem}
\smallskip

Of course it is only the implication (2)$\Rightarrow$(1) that needs proving, because (1)$\Rightarrow$(2) is a well-known classical result from \cite {BEHW} which does not require any assumptions on $\mathscr C$.

We review a precise formal setting for learnability, after which we proceed to analysis of a counter-example from \cite{DD,BEHW}. We observe that the concept class $\mathscr C$ in the example is in fact PAC learnable, and this observation provides a clue to a general result. 

The construction of the learning rule $\mathcal L$ can be described as a ``first in, first served'' approach. The concept class $\mathscr C$ is given a minimal well-ordering, $\prec$, and $\mathcal L$ is constructed recursively, by assigning to a learning sample the $\prec$-smallest consistent concept $C$ with regard to the the ordering. As a consequence,  for every concept $C\in{\mathscr C}$, the image of all learning samples of the form $(\sigma,C\cap\sigma)$ under $\mathcal L$ forms a uniform Glivenko--Cantelli class. It is for establishing this property of $\mathcal L$ that we need Martin's Axiom. Now the probable approximate correctness of $\mathcal L$ is straightforward.

The present approach goes back to present author's earlier work \cite{pestov:2010a}, but the results are new and have never been stated explicitely before.

\section{The setting}

For obvious reasons, we need to be quite precise when fixing a general setting for learnability. 
The {\em domain} ({\em instance space}) $\Omega=(\Omega,{\mathscr A})$ is a {\em standard Borel space,} that is, a complete separable metric space equipped with the sigma-algebra of Borel subsets (the smallest family of sets containing all open balls and closed under complements and countable intersections).

{\em Measures} on $\Omega$ mean {\em Borel probability measures,} that is, countably additive functions on $\mathscr A$ with values in the unit interval $[0,1]$, having the property $\mu(\Omega)=1$. We will not distinguish between a measure $\mu$ and its Lebesgue completion, that is, an extension of $\mu$ over a larger sigma-algebra of Lebesgue $\mu$-measurable subsets of $\Omega$. Furthermore, recall that a subset $A\subseteq\Omega$ is {\em universally measurable} if it is Lebesgue $\mu$-measurable for every probability measure $\mu$ on $\Omega$. 

With this caveat, a {\em concept class}, $\mathscr C$, is a family of universally measurable subsets of $\Omega$.

In the learning model, a set $\mathcal P$ of probability measures on $\Omega$ is fixed. Usually either ${\mathcal P}=P(\Omega)$ is the set of all probability measures (distribution-free learning), or ${\mathcal P}=\{\mu\}$ is a single measure (learning under fixed distribution). In our article, the case of interest is the former, although some of our results are valid in the case of a general family $\mathcal P\subseteq P(\Omega)$.

% Every probability measure $\mu$ on $\Omega$ determines an $L^1$ distance between concepts:
% \[d_{\mu}(A,B)=\mu\left(A\bigtriangleup B \right).\] 

A {\em learning sample} is a pair $(\sigma,\tau)$ of finite subsets of $\Omega$, where $\tau\subseteq\sigma$ is thought of as the set of points belonging to an unknown concept, $C$. The set of all samples of size $n$ is usually identified with $\left(\Omega\times \{0,1\}\right)^n$.

A {\em learning rule} (for $\mathscr C$) is a mapping 
\[{\mathcal L}\colon \bigcup_{n=1}^\infty\Omega^n\times \{0,1\}^n\to {\mathscr C}\]
which satisfies the following measurability condition:  for every $C\in{\mathscr C}$, $n\in\N$ and $\mu\in{\mathcal P}$, the function
\begin{equation}
\label{eq:measurabilityl}
\Omega^n\ni\sigma\mapsto \mu\left({\mathcal L}(\sigma,C\cap\sigma)\bigtriangleup C\right) \in \R 
\end{equation}
is measurable.

A learning rule $\mathcal L$ is {\em consistent} (with a concept class $\mathscr C$) if for all $C\in {\mathscr C}$, $n\in\N$ and $\sigma\in\Omega^n$ one has 
\[{\mathcal L}(\sigma,C\cap \sigma)\cap\sigma = C\cap\sigma.\]

A learning rule $\mathcal L$ is {\em probably approximately correct} ({\em PAC}) {\em under ${\mathcal P}$} if for every $\e>0$
\begin{equation}
\label{eq:pac}
\mu^{\otimes n}\left\{\sigma\in\Omega^n\colon \mu\left({\mathcal L(\sigma,C\cap\sigma)}\bigtriangleup C\right)>\e \right\} \to 0\end{equation}
as $n\to\infty$, uniformly over all $C\in {\mathscr C}$ and $\mu\in{\mathscr P}$.
Here $\mu^{\otimes n}$ denotes the product measure on $\Omega^n$.

In terms of sample complexity function $s(\e,\delta)$, a learning rule $\mathcal L$ is PAC if for each $C\in{\mathscr C}$ and every $\mu\in {\mathcal P}$ an independent identically distributed (i.i.d.) sample $\sigma=(x_1,x_2,\ldots,x_n)$ with $n\geq s(\e,\delta)$ points has the property $\mu(C\bigtriangleup {\mathcal L}(\sigma,C\cap\sigma))<\e$ with confidence $\geq 1-\delta$.

A concept class $\mathscr C$ is {\em PAC learnable} {\em under $\mathcal P$}, if there exists a PAC learning rule for $\mathscr C$ under $\mathcal P$. A class $\mathscr C$ is {\em consistently learnable} (under $\mathcal P$) if every learning rule consistent with $\mathscr C$ is PAC under $\mathcal P$. 
If $\mathcal P = P(\Omega)$ is the set of all probability measures, then $\mathscr C$ is said to be {\em distribution-free PAC learnable}. If ${\cal P}=\{\mu\}$ is a single probability measure, one is talking of {\em learning under a single distribution}.
Learnability under intermediate families $\mathcal P$ is also receiving considerable attention, cf. Chapter 7 in \cite{vidyasagar2003}.

Notice that in this paper, we only talk of {\em potential} PAC learnability, adopting a purely information-theoretic viewpoint. As a consequence, our statements about learning rules are existential rather than constructive, and 
building learning rules by transfinite recursion is perfectly acceptable.

A concept class $\mathscr C$ is {\em uniform Glivenko--Cantelli} {\em with regard to a family of measures} $\mathcal P$, if for each $\e>0$
\begin{equation}
\label{eq:glivenko}
\sup_{\mu\in {\mathcal P}}\mu^{\otimes n}\left\{\sup_{C\in{\mathscr C}}\left\vert \mu(C)-  \mu_n(C) \right\vert\geq \e\right\}\to 0\mbox{ as }n\to\infty.
\end{equation}
Here $\mu_n$ stands for the empirical (uniform) measure on $n$ points, sampled in an i.i.d. fashion from $\Omega$ according to the distribution $\mu$. In this case, one also says that $\mathscr C$ has the property of {\em uniform convergence of empirical measures} ({\em UCEM property}) (with regard to $\mathcal P$) \cite{vidyasagar2003}.

Every uniform Glivenko--Cantelli concept class (with regard to $\mathcal P$) is consistently PAC learnable (under $\mathcal P$), as is easy to verify. In the distribution-free situation ($\mathcal P=P(\Omega)$) the converse holds under additional measurability conditions on the class mentioned in the Introduction, but, as we will see, not always. 

More precisely, every distribution-free PAC learnable class has finite VC dimension (it was proved in \cite {BEHW}, Theorem 2.1(i); see also e.g. \cite{vidyasagar2003}, Lemma 7.2 on p. 279).
Now the measurabilty conditions on $\mathscr C$ assure that a class $\mathscr C$ of finite VC dimension $d$ is uniform Glivenko--Cantelli, with a sample complexity bound that does not depend on $\mathscr C$, but only on $\e$, $\delta$, and $d$. The following is a typical (and far from being optimal) such estimate, which can be deduced, for instance, along the lines of \cite{mendelson:03}:
\begin{equation}
\label{eq:standard}
s(\e,\delta,d)\leq
\frac{128}{\e^2}\left(d\log\left(\frac{2e^2}{\e}\log\frac{2e}{\e}\right) + \log\frac 8 {\delta}\right).
\end{equation}
For our purposes, we will fix any such bound and refer to it as a {\em ``standard''} sample complexity estimate for $s(\e,\delta,d)$. 

Now the consistent learnability for $\mathscr C$, with the same sample complexity, follows. Of course in order to conclude that $\mathscr C$ is  PAC learnable, it is necessary to prove the existence of a consistent learning rule satisfying Eq. (\ref{eq:measurabilityl}). This is usually being done using subtle measurable selection theorems using the same measurability assumptions on $\mathscr C$ yet again.

Finally, recall that a subset $N\subseteq\Omega$ is {\em universal null} if for every non-atomic probability measure $\mu$ on $(\Omega,{\mathscr A})$ one has $\mu(N^\prime)=0$ for some Borel set $N^\prime$ containing $N$. Universal null Borel sets are just countable sets. 

\section{Revisiting an example of Durst and Dudley\label{s:dd}}

The proof of the implication (3)$\Rightarrow$(2) in Theorem \ref{th:fundamental} depends in an essential way on the Fubini theorem, which is why some measurability restrictions on the class $\mathscr C$ are unavoidable. Without them, the conclusion is not true in general. Here is a classical example of a concept class having finite VC dimension which is not uniform Glivenko--Cantelli.

\begin{example}[Durst and Dudley \cite{DD}, Proposition 2.2]
\label{ex:dd}
Assume the validity of the Continuum Hypothesis (CH). 
Let $\Omega$ be an uncountable standard Borel space, that is, up to an isomorphism, a Borel space associated to the unit interval $[0,1]$. The statement of CH is equivalent to the existence of a total order $\prec$ on $\Omega$ with the property that every half-open initial segment $I_y=\{x\in\Omega\colon x\prec y\}$, $y\in\Omega$ is countable, and $\prec$ is a well-ordering: every non-empty subset of $\Omega$ has the smallest element.
Fix such an order.

Let $\mathscr C$ consist of all half-open initial segments $I_y$, $y\in\Omega$ as above. Clearly, the VC dimension of the class $\mathscr C$ is one. 

Now let $\mu$ be a non-atomic Borel probability measure on $\Omega$ (e.g., the Lebesgue measure on $[0,1]$). Under CH, every element of $\mathscr C$ is a countable set, therefore Borel measurable of measure zero. 
At the same time, for every $n$ and each i.i.d. random $n$-sample $\sigma$, there is a countable initial segment $C=I_y\in {\mathscr C}$ containing all elements of $\sigma$. The empirical measure of $C$ with regard to $\sigma$ is one. Thus, no finite sample guesses the measure of all elements of $\mathscr C$ to within an accuracy $\e<1$ with a non-vanishing confidence.
\end{example}

See also \cite{WD}, p. 314; \cite{dudley}, pp. 170--171.

A further modification of this construction gives an example of a concept class of finite VC dimension which is not consistently PAC learnable.

\begin{example}[Blumer {\em et al.} \cite{BEHW}, p. 953]
\label{ex:behw}
Again, assume CH.
Add to the concept class $\mathscr C$ from Example \ref{ex:dd} the set $\Omega$ as an element, forming a new concept class ${\mathscr C}^\prime={\mathscr C}\cup\{\Omega\}$. One still has $\VC(\mathscr C^\prime)=1$.
For a finite labelled sample $(\sigma,\tau)$ define
\begin{equation}
\label{eq:rule}
{\mathcal L}(\sigma,\tau) = I_z,~~z= \min\{y\in(\Omega,\prec)\colon \tau\subseteq I_y\}.\end{equation}
The learning rule $\mathcal L$ is consistent with the class ${\mathscr C}^\prime$. At the same time, $\mathcal L$ is not probably approximately correct. Indeed, for the concept $C=\Omega$ the value of the learning rule ${\mathcal L}(\sigma,\Omega\cap\sigma)={\mathcal L}(\sigma,\sigma)$ will always return a countable concept $I_y$ for some $y\in\Omega$, and if $\mu$ is a non-atomic Borel probability measure on $\Omega$, then $\mu(C\bigtriangleup I_y)=1$. The concept $C=\Omega$ cannot be learned to accuracy $\e<1$ with a non-zero confidence.
\end{example}

\begin{remark}
It is important to note that --- again, under CH --- the class ${\mathscr C}^\prime$ is distribution-free PAC learnable. 
\end{remark}

Indeed, redefine a well-ordering on $\mathscr C=\{I_x\colon x\in\Omega\}\cup\{\Omega\}$ by making $\Omega$ the smallest element (instead of the largest one) and keeping the order relation between other elements the same. Denote the new order relation by $\prec_1$, and define a learning rule $\mathcal L_1$ similarly to Eq. (\ref{eq:rule}), but this time understanding the minimum with regard to the well-ordering $\prec_1$:
\begin{equation}
\label{eq:rule1}
{\mathcal L}_1(\sigma,\tau) = \min_{(\prec_1)}\left\{C\in{\mathscr C}\colon C\cap\sigma = \bigcap_{\tau\subseteq D}D\right\}.\end{equation}
In essence, $\mathcal L_1$ examines all the concepts following a transfinite order on them, and returns the first encountered concept consistent with the sample, provided it exists.

To see what difference it makes with Example \ref{ex:behw}, let $\mu$ be again a non-atomic probability measure on $\Omega$. If $C=\Omega$, then for every sample $\sigma$ consistently labelled with $C$ the rule $\mathcal L_1$ will return $C$, because this is the smallest consistent concept encountered by the algorithm. If $C\neq\Omega$, then for $\mu$-almost all samples $\sigma$ the labelling on $\sigma$ produced by $C$ will be empty, and the concept ${\mathcal L}_1(\sigma,\emptyset)$ returned by $\mathcal L_1$, while possibly different from $C$, will be again a countable concept, meaning that $\mu(C\bigtriangleup{\mathcal L}(\sigma,\emptyset))=0$. 

To give a formal proof that $\mathcal L_1$ is PAC, notice that for every $C\in {\mathscr C}^\prime$ and each $n\in\N$ the collection of pairwise distinct concepts ${\mathcal L}_1(\sigma\cap C)$, $\sigma\in\Omega^n$ is only countable (under CH), because they are all contained in the $\prec_1$-initial segment of a minimally ordered set $\mathscr C$ of cardinality continuum, bounded by $C$ itself. As a consequence, the concept class
\begin{equation}
\label{eq:lc}
{\mathcal L}_1^{C}=\{{\mathcal L}_1(\sigma\cap C)\colon\sigma\in\Omega^n,n\in\N\}\subseteq {\mathscr C}^\prime\end{equation}
is also countable (assuming CH). The VC dimension of the family ${\mathcal L}_1^{C}\cup\{C\}$ is $\leq 1$, and being countable, it is a uniform Glivenko--Cantelli class with a standard sample complexity as in Eq. (\ref{eq:standard}). Consequently, given $\e,\delta>0$, and assuming that $n$ is sufficiently large, one has for each probability measure $\mu$ on $\Omega$ and every $\sigma\in\Omega^n$
\[\mu(C\bigtriangleup{\mathcal L}(\sigma,C\cap\sigma))<\e\]
provided $n\geq s(\e,\delta,1)$,
as required.

\begin{remark}
Notice that the role of the Continuum Hypothesis in the above examples was merely to assure that every initial segment $I_y$, $y\in\Omega$ is a universally measurable set. 
As we will see, it can be achieved under a much milder assumption of Martin's Axiom.
\end{remark}

\begin{remark}
\label{r:notalways}
Thus, under the Continuum Hypothesis, the example of Dudley and Durst as modified by Blumer, Ehrenfeucht, Haussler, and Warmuth gives an example of a PAC learnable concept class which is not uniform Glivenko--Cantelli (even if having finite VC dimension). As it will become clear in the next Section, the assumption of CH can be weakened to Martin's Axiom. Still, it would be interesting to know whether an example with the same combination of properties can be constructed without additional set-theoretic assumptions.
\end{remark}

A basic observation of this Section is that in order for a learning rule $\mathcal L$ to be PAC, the assumption on $\mathscr C$ being uniform Glivenko--Cantelli can be weakened as follows.

\begin{lemma}
\label{l:basic}
Let $\mathscr C$ be a concept class and $\mathcal P$ a family of probability measures on the domain $\Omega$. Suppose there exists a function $s(\e,\delta)$ and a learning rule $\mathcal L$ for $\mathscr C$ with the property that for every $C\in{\mathscr C}$, the set ${\mathcal L}^{C}\cup\{C\}$ is Glivenko--Cantelli with regard to $\mathcal P$ with the sample complexity $s(\e,\delta)$, where
\[{\mathcal L}^{C}=\left\{{\mathcal L}(C\cap \sigma)\colon\sigma\in\Omega^n,n\in\N\right\}.\]
Then $\mathcal L$ is probably approximately correct under $\mathcal P$ with sample complexity $s(\e,\delta)$. \hfill\QED
\end{lemma}

This simple fact becomes useful in combination with the technique of well-orderings.
Of course the Continuum Hypothesis is a particularly unnatural assumption in a probabilistic context (cf. \cite{freiling}). But it is unnecessary.
Martin's Axiom (MA) is a much weaker and natural additional set-theoretic axiom, which works just as well. 
% We explain how the above idea is formalized in the setting of Martin's Axiom in the next Section.

\section{Learnability under Martin's Axiom}

Martin's Axiom (MA) says that no compact Hausdorff topological space with the countable chain condition is a union of strictly less than continuum nowhere dense subsets. Thus, it is a stronger statement than the Baire Category Theorem. In particular, the Continuum Hypothesis implies MA. However, MA is compatible with the negation of CH, and this is where the most interesting applications of MA are to be found. We need the following consequence of MA.

\begin{theorem}[Martin-Solovay] 
\label{th:martin-solovay}
Let $(\Omega,\mu)$ be a standard Lebesgue non-atomic probability space. 
Under MA, 
the Lebesgue measure is $2^{\aleph_0}$-additive, that is, if $\kappa<2^{\aleph_0}$ and $A_{\alpha}$, $\alpha<\kappa$ is family of pairwise disjoint measurable sets, then $\cup_{\alpha<\kappa}A_{\alpha}$ is Lebesgue measurable and
\[\mu\left(\bigcup_{\alpha<\kappa}A_{\alpha} \right) = \sum_{\alpha<\kappa}\mu(A_{\alpha}).\]
In particular, the union of strictly less than continuum null subsets of $\Omega$ is a null subset. \hfill
\QED
\end{theorem}

For the proof and more on MA, see \cite{kunen}, Theorem 2.21, or \cite{fremlin}, or \cite{jech}, pp. 563--565.

\begin{lemma}
\label{l:fma}
Let $\mathscr C$ be a concept class  and $\mathcal P$ a family of probability measures on a standard Borel domain $\Omega$. Consider the following properties.
\begin{enumerate}
\item \label{fma:1} Every countable subclass of $\mathscr C$ is uniform Glivenko--Cantelli with regard to $\mathcal P$.
\item \label{fma:2} There is a function $s(\e,\delta)$ so that every countable subclass of $\mathscr C$ is uniform Glivenko--Cantelli with regard to $\mathcal P$ with sample complexity $s(\e,\delta)$.
\item \label{fma:3} Every subclass ${\mathscr C}^\prime$ of $\mathscr C$ having cardinality $<2^{\aleph_0}$ is uniform Glivenko--Cantelli with regard to $\mathcal P$.
\item \label{fma:4} There is a function $s(\e,\delta)$ so that every subclass ${\mathscr C}^\prime$ of $\mathscr C$ having cardinality $<2^{\aleph_0}$ is uniform Glivenko--Cantelli with regard to $\mathcal P$ with sample complexity $s(\e,\delta)$.
\end{enumerate}
Then  
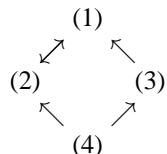
\begin{center}
(\ref{fma:1}) \\
$\neswarrow\phantom{xx}\nwarrow$ \\
(\ref{fma:2}) $\phantom{xxxxx}$ (\ref{fma:3})\\
$\nwarrow\phantom{xx}\nearrow$ \\
(\ref{fma:4})
\end{center}
Under Martin's Axiom, all four conditions are equivalent.
\end{lemma}

\begin{proof}
The implications $(\ref{fma:2})\Rightarrow (\ref{fma:1})$, $(\ref{fma:3})\Rightarrow (\ref{fma:1})$, $(\ref{fma:4})\Rightarrow (\ref{fma:2})$ and $(\ref{fma:4})\Rightarrow (\ref{fma:3})$ are trivially true. To show $(\ref{fma:1})\Rightarrow (\ref{fma:2})$, let $\delta,\e>0$ be artitrary but fixed. For each countable subclass ${\mathscr C}^\prime$, choose the smallest value of sample complexity $s=s({\mathscr C}^\prime,\e,\delta)$. The function ${\mathscr C}^\prime\mapsto s({\mathscr C}^\prime,\e,\delta)$ is monotone under inclusions: if ${\mathscr C}^\prime\subseteq {\mathscr C}^{\prime\prime}$, then $s({\mathscr C}^\prime,\e,\delta)\leq s({\mathscr C}^{\prime\prime},\e,\delta)$. If ${\mathscr C}^\prime_n$ is a sequence of countable classes, then the union $\cup_{n=1}^\infty {\mathscr C}^\prime_n$ is a countable class, whose sample complexity value bounds from above $s({\mathscr C}^\prime,\e,\delta)$, $n=1,2,\ldots$. Thus, the function ${\mathscr C}^\prime\mapsto s({\mathscr C}^\prime,\e,\delta)$ for $\delta,\e>0$ fixed is bounded on countable sets of inputs, and therefore bounded. 

Now assume (MA). It is enough to prove $(\ref{fma:2})\Rightarrow (\ref{fma:4})$. This is done by a transfinite induction on the cardinality $\kappa=\abs{{\mathscr C}^\prime}$, which never exceeds $2^{\aleph_0}$ because ${\mathscr C}^\prime$ consists of Borel subsets of a standard Borel domain. For $\kappa=\aleph_0$ there is nothing to prove. Else, represent $\mathscr C$ as a union of an increasing transfinite chain of concept classes ${\mathscr C}_{\alpha}$, $\alpha<\kappa$, for each of which the statement of (\ref{fma:4}) holds. For every $\e>0$ and $n\in\N$, the set 
\[\begin{array}{l}\left\{\sigma\in\Omega^n\colon \sup_{C\in{\mathscr C}}\left\vert\mu_n(\sigma)(C)-\mu(C)\right\vert <\e\right\}\\[3mm] = \bigcap_{\alpha<\kappa} \left\{\sigma\in\Omega^n\colon \sup_{C\in{\mathscr C}_{\alpha}}\left\vert\mu_n(\sigma)(C)-\mu(C)\right\vert <\e\right\}
\end{array}
\]
is measurable by Martin-Solovay's Theorem \ref{th:martin-solovay}. Given $\delta>0$ and $n\geq s(\e,\delta,d)$, another application of the same result leads to conclude that for every $\mu\in P(\Omega)$:
\begin{eqnarray*}
&&\mu^{\otimes n}\left\{\sigma\in\Omega^n\colon \sup_{C\in{\mathscr C}}\left\vert\mu_n(\sigma)(C)-\mu(C)\right\vert <\e\right\} \\ &=& 
\mu^{\otimes n}\left(
\bigcap_{\alpha<\kappa} \left\{\sigma\in\Omega^n\colon \sup_{C\in{\mathscr C}_{\alpha}}\left\vert\mu_n(\sigma)(C)-\mu(C)\right\vert <\e\right\}\right) \\
&=& \inf_{\alpha<\kappa} \mu^{\otimes n}\left\{\sigma\in\Omega^n\colon \sup_{C\in{\mathscr C}_{\alpha}}\left\vert\mu_n(\sigma)(C)-\mu(C)\right\vert <\e\right\}\\
&\geq& 1-\delta,
\end{eqnarray*}
as required.
\end{proof}

\begin{lemma}
\label{l:mapac}
Let $\mathscr C$ be a concept class whose countable subclasses are uniform Glivenko--Cantelli with regard to a family of probability measures $\mathcal P$. Let $\mathcal L$ be a learning rule for $\mathscr C$ with the property that for every $C\in{\mathscr C}$, the set
\begin{eqnarray}
\label{eq:l}
{\mathcal L}^{C,n}=\left\{{\mathcal L}(C\cap\sigma)\colon\sigma\in\Omega^n\right\}\end{eqnarray}
has cardinality strictly less than continuum. Under Martin's Axiom, the rule $\mathcal L$ is probably approximately correct under $\mathcal P$. The common sample complexity bound of countable subclasses of $\mathscr C$ becomes the sample complexity bound for the learning rule $\mathcal L$. 
\end{lemma}

\begin{proof}
Recall that $2^{\aleph_0}$ is a regular cardinal, and thus admits no countable cofinal subset. Therefore, under the assumptions of Lemma, the cardinality of ${\mathcal L}^{C}=\cup_{n=1}^{\infty}{\mathcal L}^{C,n}$ is still strictly less than continuum. Applying now Lemma \ref{l:fma} and then Lemma \ref{l:basic}, we conclude.
\end{proof}

The following result establishes existence of learning rules with the required property.

\begin{lemma}
\label{l:l}
Let $\mathscr C$ be an infinite concept class on a measurable space $\Omega$. Denote $\kappa=\abs{\mathscr C}$ the cardinality of $\mathscr C$. There exists a consistent learning rule $\mathcal L$ for $\mathscr C$ with the property that for every $C\in {\mathscr C}$ and each $n$, the set ${\mathcal L}^{C,n}$ (cf. Eq. (\ref{eq:l}))
% \begin{eqnarray}
% \label{eq:l}
% {\mathcal L}^{f,n}=\left\{{\mathcal L}(f\vert\sigma)\colon\sigma\in\Omega^n\right\}\end{eqnarray}
has cardinality $<\kappa$. Under MA the rule $\mathcal L$ satisfies the  condition in Eq. (\ref{eq:measurabilityl}).
\end{lemma}

\begin{proof}
Choose a minimal well-ordering of elements of $\mathscr C$:
\[{\mathscr C}=\{C_{\alpha}\colon\alpha<\kappa\},\]
and set for every $\sigma\in\Omega^n$ and $\tau\in \{0,1\}^n$ the value   ${\mathcal L}(\sigma,\tau)$ equal to $C_{\beta}$, where 
\[\beta = \min\{\alpha<\kappa\colon C_{\alpha}\cap\sigma=\tau\},\]
provided such a $\beta$ exists. 
Clearly, for each $\alpha<\kappa$ one has
\[{\mathcal L}(\sigma,C_{\alpha}\cap\sigma)\in
\{C_{\beta}\colon\beta\leq\alpha\},\]
which assures (\ref{eq:l}). 
Besides, the learning rule $\mathcal L$ is consistent.

Fix $C=C_{\alpha}\in {\mathscr C}$, $\alpha<\kappa$. For every $\beta\leq\alpha$ define $D_{\beta}=\{\sigma\in\Omega^n\colon C\cap\sigma = C_{\beta}\cap\sigma\}$. The sets $D_{\beta}$ are measurable, and the function \[\Omega^n\ni\sigma\mapsto \mu({\mathcal L}(C\cap \sigma)\bigtriangleup C)\in\R\]
takes a constant value $\mu(C\bigtriangleup C_{\alpha})$ on each set $D_{\beta}\setminus\cup_{\gamma<\beta}D_{\gamma}$, $\beta\leq\alpha$. Such sets, as well as all their possible unions, are measurable under MA by force of Martin--Solovay's Theorem \ref{th:martin-solovay}, and their union is $\Omega^n$. This implies the validity of Eq. (\ref{eq:measurabilityl}) for $\mathcal L$.
\end{proof}

Lemma \ref{l:mapac} and Lemma \ref{l:l} lead to the following result.

\begin{theorem}[Assuming MA]
\label{th:countablesubclassesugc}
Let $\mathscr C$ be a concept class consisting of Borel measurable subsets of a standard Borel domain $\Omega$, and let $\mathcal P$ be a family of probability measures on $\Omega$. Suppose that every countable subclass of $\mathscr C$ is uniform Glivenko--Cantelli with regard to $\mathcal P$. Then the concept class $\mathscr C$ is PAC learnable under $\mathcal P$. In addition, there exists a common sample complexity bound for countable subclasses of $\mathscr C$, and any such bound  gives a sample complexity bound for PAC learnability of $\mathscr C$. \hfill \QED
\end{theorem}

Finally, we can deduce our main result.

\subsection*{Proof of (2)$\Rightarrow$(1) in Theorem \ref{th:main}}
% 
% (1)$\Rightarrow$(2): suppose $\mathscr C$ is PAC learnable, and let $\mathcal L$ be a distribution-free PAC learning rule for $\mathscr C$. Then every countable subset $\mathscr C^\prime$ of $\mathscr C$ is PAC learnable by the restriction of the learning rule $\mathcal L$ to $\mathscr C^\prime$, and consequently has finite VC dimension. The function ${\mathscr C}^\prime\mapsto \VC({\mathscr C})$ is monotone with regard to inclusion, and consequently has to be bounded uniformly over all $\mathscr C^\prime$. The common bound gives the value of VC dimension for $\mathscr C$.
% \par
The implication follows from Theorem \ref{th:countablesubclassesugc} with ${\mathcal P}=P(\Omega)$ and the common complexity bound (\ref{eq:standard}).
\hfill\QED

\newpage

\section{Conclusion}

As a footnote to the fundamental theorem of statistical learing, we have proved that in the presence of a mild set-theoretic axiom (Martin's Axiom), PAC learnability of a concept class $\mathscr C$ is equivalent to finiteness of VC dimension of $\mathscr C$, without any extra assumptions on the measurability of the class $\mathscr C$. The price to pay is giving up consistent PAC learnability, as well as constructive choice of a learning rule.

It would be interesting to know to what extent the results remain true in the usual ZFC model of set theory. In particular, can an example of a concept class $\mathscr C$ on a standard Borel domain which has finite VC dimension and still is not cosistently PAC learnable, be constructed without any additional set-theoretic axioms?

\end{document}